\documentclass{amsart}
\usepackage{amsmath}
\usepackage{amsthm}
\usepackage{amssymb}
\usepackage{graphicx}
\usepackage{geometry}
\usepackage{algorithm}
\usepackage{algorithmic}
\usepackage[latin1]{inputenc}    
\usepackage[OT1]{fontenc}
\usepackage[english]{babel}
\usepackage{mathrsfs}
\usepackage{dsfont}
\usepackage{comment}
\usepackage{url}

\def\quickop#1{\expandafter\DeclareMathOperator\csname
#1\endcsname{#1}}

\quickop{dis}\quickop{diam}

\newcommand{\GWREC}{\texttt{MREC}}
\newcommand{\GW}{\mathrm{GW}}
\newcommand{\R}{\mathbb{R}}
\newcommand{\Prob}{\mathbb{P}}
\newcommand{\bC}{{\mathbb{C}}}
\newcommand{\bR}{{\mathbb{R}}}
\newcommand{\bM}{{\mathbb{M}}}

\newtheorem{thm}{Theorem}[section]
\newtheorem{cor}[thm]{Corollary}
\newtheorem{prop}[thm]{Proposition}

\newtheorem{defin}[thm]{Definition}

\begin{document}

\title[MREC: framework for matching point clouds with applications to single-cell]{MREC: a fast and versatile framework for aligning and matching point clouds with applications to single cell molecular data}
\author[A.J. Blumberg]{Andrew J. Blumberg}
\author[M. Carri\`ere]{Mathieu Carri\`ere}
\author[M.A. Mandell]{Michael A. Mandell}
\author[R. Rabad\'an]{Ra\'ul Rabad\'an}
\author[S. Villar]{Soledad Villar}

\begin{abstract}
Comparing and aligning large datasets is a pervasive problem occurring across many different knowledge domains.
We introduce and study \GWREC, a recursive decomposition algorithm for
computing matchings between datasets.  The basic idea is to partition
the data, match the partitions, and then recursively match the points
within each pair of identified partitions.  The matching itself is
done using black box matching procedures that are too expensive to run
on the entire dataset.  Using an absolute measure of the quality of a
matching, the framework supports optimization over parameters
including partitioning procedures and matching algorithms.  By design,
\GWREC\ can be applied to extremely large datasets.  We analyze the
procedure to describe when we can expect it to work well and
demonstrate its flexibility and power by applying it to a number of
alignment problems arising in the analysis of single cell molecular data.
\end{abstract}
\maketitle

\section{Introduction}

Computing alignments, or matchings, between two point clouds is a
basic problem in data analysis.  Notable applications include aligning
points between images (sometimes referred to as the ``point
registration" problem)~\cite{kimmel03, memoli05}, measuring distances
between images or histograms~\cite{Rubner2000}, finding independent
edge sets in bipartite graphs~\cite{Bougleux2017}, matching the shape
of primate teeth to determine species
relationships~\cite{boyer2011algorithms}, removing batch effect in
single-cell data~\cite{Forrow2019}, and studying stem cell
differentiation~\cite{Forrow2019, Schiebinger2019}. In the setting
where the point clouds correspond to 3D shapes or low-dimensional
manifolds, there exist techniques that successfully exploit the
geometric structure of the point clouds~\cite{lipman2011conformal,
  gao2018development, bajaj2018smac}. In this work, we focus on
arbitrary metric spaces, where the matching problem is typically
formalized as an optimization problem: for each matching of datasets
$X$ and $Y$, by which we mean a pair of maps $f \colon X \to Y$
and $g \colon Y \to X$, we assign a cost $c(f,g)$.  The goal is then
to find a matching that minimizes the cost.  The cost of the optimal
matching provides a measure of some kind of distance between $X$ and $Y$;
in this formulation, there is a tight connection between matchings and
distances between datasets.

{\bf Matchings and distances between metric spaces:}
A standard way of assigning costs to the matching is to assume that
the two point clouds $X$ and $Y$ are equipped with distance functions
$\partial_X$ and $\partial_Y$ respectively, i.e., they are metric spaces.  Then the
cost of a matching can be defined in terms of relationships between
the interpoint distances.  This approach to matching is closely
related to notions of distances between metric spaces known as the
{\em Hausdorff} and {\em Gromov-Hausdorff} distances~\cite{gromov81, gromov}.  A distinct
advantage of cost measures that depend only on the distance matrices
for each dataset is that the resulting alignment procedures are
invariant with respect to rotations and more generally arbitrary
distance-preserving transformations of the point clouds.  This is
clearly an essential feature when working with images, for example.
Unfortunately, computing these distances in practice is infeasible.

A closely related viewpoint is to study matching questions from the
perspective of ``earthmover distances'', which arise from an area of
mathematics known as {\em optimal transport}~\cite{Villani2008}.  Here
the idea is that matchings can be probabilistic or fuzzy, which is
encoded by representing the datasets as having probability
distributions.  The costs of the matchings are then computed by adding up
the amounts of probability mass required to transform one distribution
into another, weighted by the distances.  The resulting distances
between metric spaces are known as the {\em Wasserstein} and {\em
  Gromov-Wasserstein} distance~\cite{memoli11,Sturm}.  This approach is related to the
approach above insofar as the Gromov-Wasserstein distance is a
relaxation of the Gromov-Hausdorff distance.  Although there is a
substantial and beautiful theoretical literature on optimal transport,
the Gromov-Wasserstein distance has until recently been too expensive
to use in practice, limited to at most a few hundred points.  A
related relaxation in terms of semidefinite programming yields a
distance and matching procedure that has attractive theoretical
properties but is also expensive to use in practice~\cite{Rabin2011,
  Shirdhonkar2008, kezurer2015tight, Villar2016}; again, it is limited to hundreds of
points.

Quite recently, there has been exciting progress in the area based on
a relaxation of the Gromov-Wasserstein distance that involves
regularizing the optimal transport problem by adding an entropy
term~\cite{Cuturi2013}.  Roughly speaking, this adds a penalty based
on the complexity of the matching.  This refinement to the 
minimization objective (see Section~\ref{subsec:GW}) turns the problem
into one that can be solved comparatively efficiently with
Sinkhorn's fixed-point algorithm~\cite{Sinkhorn1974}. However,
applying this approximation to datasets with more than a few thousand
points is only feasible in low dimensions.

The limitations on the applicable size of
datasets is a significant issue for many potential applications, including
single-cell characterization of biological systems.  Recent technological breakthroughs have allowed 
characterization of gene expression, chromatin accessibility, and structure at the single-cell level, allowing
biologists to generate datasets with tens of thousands of cells (data points) in a very high-dimensional space (each dimension represents a gene of
interest). These technical limitations are unfortunate since optimal transport has been shown 
in recent work to have significant potential applications,
e.g., for removing batch effects~\cite{Forrow2019}, or for a better
understanding of the temporal evolution and differentiation in development~\cite{Schiebinger2019}. 

{\bf Contributions.}
Our goal in this article is to introduce and study a general framework
for applying matching algorithms to very large datasets by {\em
  recursive decomposition}.  The basic idea is to partition the dataset 
  into (possibly overlapping) subsets, match representative points
from the subsets using any matching algorithm as a subroutine, and
then recursively match the partitions.  We refer to the recursive
approximation scheme as \GWREC\ (see Algorithm~\ref{alg:gwrec}).  By
selecting parameters appropriately, we can work with datasets of
tens of thousands of points.  A particularly
attractive features of this framework is that, since there is an
external measure of the quality of a matching, we can efficiently
search over the parameter space, including selection of matching
algorithms as well as methods for partitioning. 

After introducing the algorithm, we do a simple theoretical analysis
that explains situations when we can guarantee that this approximation
will work well (see Section~\ref{sec:theo}).  The conclusions are
reassuring but unsurprising: the algorithm can be expected
to work well either when the matching itself admits a recursive
decomposition (e.g., the data is produced by well-separated Gaussians)
or the partitions produce an approximation that is close to the
original metric space.

Finally, we explore a number of applications of \GWREC\  to datasets coming from single-cell characterization of different biological systems, showcasing the usefulness of \GWREC\ for
various problems in this field (see Section~\ref{sec:appli}).  The
results we obtain outperform or are comparable to alternative
algorithms on small datasets, and produce interesting results on datasets 
too large to be amenable to any other techniques.

\section{Matchings and Gromov-type distances}
\label{sec:GH}

In this section, we recall the definitions of matchings and various
matching distances between metric spaces. We first review the
Gromov-Hausdorff distance in Section~\ref{subsec:GH}.  We present
its probabilistic (and entropy-regularized) relaxation, the
Gromov-Wasserstein distance, in Section~\ref{subsec:GW}, and we
discuss the semidefinite relaxation in Section~\ref{subsec:SDP}.

{\bf Problem.} We can formulate the basic problem of interest as
follows.  We are given two datasets $X$ and $Y$, each equipped with a
distance function.  That is, we have functions $\partial_X \colon X
\times X \to \bR$ and $\partial_Y \colon Y \times Y \to \bR$ that
satisfy the metric axioms. 
The problem is to find a matching that best preserves the distances.
There are of course many ways to make this precise, depending on the
notion of ``best''. Our point of departure is the solution that comes
from the idea of the Gromov-Hausdorff distance~\cite{gromov81, gromov}. 

\subsection{The Gromov-Hausdorff distance}
\label{subsec:GH}

The Gromov-Hausdorff distance is a metric on the set of compact metric
spaces up to isometry, which generalizes the Hausdorff distance.

\begin{defin} Let $A$ and $B$ be subsets of a metric space $X$.  The Hausdorff distance between $A$ and $B$ is defined as
$d_{H}(A,B) = \max\left\{\max_a d(a,B), \max_b d(b,A) \right\}$,
where $d(a,B)=\min_b d(a,b)$ and $d(b,A)=\min_a d(a,b)$.
\end{defin}

The Hausdorff distance depends on the embedding of $A$ and $B$, and
for instance is not translation or rotation invariant for $A$ and $B$
separately.  As a consequence, it is useful to consider a refinement
which is invariant to distance-preserving transformations. \\

{\bf Correspondences.}  Let $R$ be a correspondence between
$X$ and $Y$, i.e., a subset $R \subset X \times Y$ such that the
projections are surjective onto $X$ and $Y$.  A correspondence is a
kind of matching where multiple points can be associated with each
other.  Then, the {\em distortion} of $R$ is defined to be: 
$\dis(R) = \sup |\partial_X(x,x') - \partial_Y(y,y')|$,
%
where the supremum is over all $(x,y) \in R$ and $(x',y') \in R$.
The distortion can be thought of as the cost of the correspondence:
the quality of the matching is assessed by how well distances between
matched points are preserved. \\ 

{\bf Gromov-Hausdorff distance.} 
The Gromov-Hausdorff distance is defined as:
\[
d_{GH}((X,\partial_X), (Y, \partial_Y)) = \frac{1}{2} \inf_{R} \dis(R),
\]
where the infimum varies over all correspondences.  That is, in
principle, if we compute the Gromov-Hausdorff distance, we also obtain
a correspondence which realizes it. 
Unfortunately, it is infeasible to directly compute the
Gromov-Hausdorff distance.  The optimization problem above is an
example of an integer programming problem, and is well-known to be
NP-hard, even for linear spaces such as metric
trees~\cite{Agarwal2018}.  Instead, a standard approach is to turn to
various kinds of relaxations.

\subsection{The Gromov-Wasserstein distance}
\label{subsec:GW}

A first kind of relaxation is to allow probabilistic matchings; each
point $x \in X$ is matched to a point $y \in Y$ with a probability
$p_{xy}$, subject to the constraint that $\sum_{y_i \in Y} p_{x y_i} =
1$.  It turns out that a way to study this is via the
consideration of metric spaces equipped with a probability
measure. \\

{\bf Metric probability space.} We will temporarily broaden
consideration from metric spaces to {\em metric measure spaces}; these
are metric spaces $(X, \partial_X)$ equipped with a Borel measure
$\Prob_X$.  In fact, we restrict attention to {\em metric probability
  spaces}, i.e., metric measure spaces where the total measure is $1$.
Any metric space can be regarded as a metric measure space equipped
with the uniform measure, where each point is assigned equal
probability.   To construct a metric in this setting which is analogous to the Gromov-Hausdorff distance, we need to introduce a
probabilistic analogue of a correspondence.  The correct notion here
is that of a {\em coupling} of $\Prob_X$ and $\Prob_Y$; this is a
probability measure $\Prob$ on $X \times Y$ such that $\Prob(\cdot, Y)
= \Prob_X$ and $\Prob(X,\cdot) = \Prob_Y$.  We can now define the
Gromov-Wasserstein distance~\cite{memoli11, Sturm, LottVillani}. \\

{\bf Gromov-Wasserstein distance.} 
Let $\Gamma_{x,x',y,y'}$  denote $|\partial_X(x,x')-\partial_Y(y,y')|$.

\begin{defin}\label{def:GW}
Let $p\in \mathbb N$, $p>0$. The $p$-th {\em Gromov-Wasserstein
  distance} $\GW_p(X,Y)$ computed between two metric probability spaces $(X,\partial_X,\Prob_X)$ and $(Y,\partial_Y,\Prob_Y)$ is defined as: 
\begin{equation*}
\underset{\Prob}{\rm inf} \left(\int\!\int \Gamma_{x,x',y,y'}^p\,\Prob(dx,dy)\,\Prob(dx',dy')\right)^{1/p},
\end{equation*}
where $\Prob$ ranges over couplings of $\Prob_X$ and $\Prob_Y$.  If we
replace $\Gamma_{x,x',y,y'}$ by $\partial_Z(x,y)$ when $X,Y\subset Z$, 
the distance is called the {\em Wasserstein} distance;
this is analogous to the Hausdorff distance.
\end{defin}

Intuitively, the Gromov-Wasserstein distance is computed by testing every probabilistic matching $\Prob$ between $X$ and $Y$. Each such matching assigns, for a given $x\in X$, a probability distribution of matching over $Y$, and vice-versa. 
The matching that best minimizes the so-called {\em metric distortion score} between the spaces is picked and its score is used as the distance value.  By rounding, one can extract an absolute (non-probabilistic) matching.

Unfortunately, the Gromov-Wasserstein is again very hard to compute in
general; this is a quadratic assignment problem (QAP), which are in
general known to be hard to solve or even provably
approximate~\cite{pardalos1994quadratic}. \\

{\bf Entropic regularization.} 
%
Since the Gromov-Wasserstein distance is still expensive to compute,
attention has been focused recently on a further relaxation, which adjusts the cost function to add a penalty term in terms of the entropy of the couplings~\cite{Cuturi2013},
where the {\em entropy} of a coupling $\Prob$ is
$H(\Prob) = - \int_{X \times Y} \ln(\Prob(x,y)) \Prob(dx,dy)$.

\begin{defin}
For a fixed $\epsilon > 0$, the entropy-regularized Gromov-Wasserstein
distance $\GW^\epsilon_p(X,Y)$ is: 
\begin{align*}
\underset{\Prob}{\rm inf} 
\left( \int\int \Gamma_{x,x',y,y'}^p\,\Prob(dx,dy)\,\Prob(dx',dy') \right)^{1/p} + \epsilon H(\Prob). 
\end{align*}
\end{defin}

Adding the entropy term turns out to result in a problem that
can be efficiently approximated with, e.g., Sinkhorn's fixed-point
algorithm~\cite{Sinkhorn1974}.  However, solving this problem on
thousands of points is feasible only in very low dimensions.

\subsection{Relaxation of Gromov-Hausdorff
  distance}
\label{subsec:SDP}

A different relaxation of the Gromov-Hausdorff distance is given by
semidefinite programming~\cite{Villar2016}.  This has the advantage that the resulting
optimization problem is convex, comes with a certificate of
optimality, and yields a pseudometric that satisfies the triangle
inequality.  On the other hand, the result is a semidefinite program in $n^4$ variables. A further relaxation reduces its complexity by introducing sparsity assumptions \cite{ferreira2018semidefinite}, but it is still infeasible for large datasets.  

Let $n$ and $m$ be the number of points of $X$ and $Y$ respectively. For the semidefinite relaxation, we optimize over a matrix $\hat {\mathbf
  Z} \in \R^{nm \times nm}$.  In fact, we work with an augmented
matrix $\mathbf Z \in \R^{nm+1 \times nm+1}$, with entries of $\mathbf
Z$ indexed by pairs $(ij,i'j')$,$(ij,nm+1)$, $(nm+1, i'j')$ and
$(nm+1,nm+1)$ with $i,i'=1,\ldots n$ and $j,j'=1,\ldots,m$,
\begin{equation}
\label{Z}
\mathbf Z=\left[ \begin{array}{c c} \hat{\mathbf Z }& \mathbf z
    \\ \mathbf z^\top &1 \end{array} \right].
\end{equation}

\begin{defin}
The SDP relaxation of the Gromov-Hausdorff distance is the solution to
the following semidefinite programming problem:
\[
\tilde d_{\mathcal{A},p}(X,Y)=\frac12  \left( \frac1{n^2}\min_{\mathbf Z}\operatorname{Trace}(\Gamma^{(p)}{\hat{\mathbf Z}} ) \right)^{1/p}
\]
subject to $\mathbf Z \in \mathcal{ A}$, where $\mathcal{A}$ denotes the set of matrices satisfying the
constraints: 
$\sum_{i}{\mathbf Z}_{ij,nm+1}\geq 1$, $\sum_{j}{\mathbf Z}_{ij,nm+1}\geq 1$, 
$\sum_{i,i'} {\mathbf Z}_{ij,i'j'}\geq 1$, $\sum_{j,j'} {\mathbf Z}_{ij,i'j'}\geq 1$, 
$\mathbf{ \hat Z 1}=\max\{n,m\} \mathbf z$, $0\leq {\mathbf Z}\leq 1$, and ${\mathbf Z}$ 
is symmetric and positive
semidefinite.
\end{defin}

As noted above, one of the most interesting properties of the SDP
relaxation is that it yields a distance that also satisfies the
triangle inequality; this is somewhat surprising, as there is not an
obvious geometric interpretation as with the optimal transport
framework.

\section{A recursive approximation scheme} 
\label{sec:recu}
In this section, we define and study a recursive approximation scheme
for computing matchings between datasets. This algorithm, that we call
\GWREC\ and define in Algorithm~\ref{alg:gwrec}, requires a black box
matching function (such as the entropy-regularized
Gromov-Wasserstein distance or the SDP relaxation presented in
Section~\ref{subsec:GW}) and a black box clustering function (such as
$K$-means++, \cite{arthur2007k}) as parameters, and uses them recursively to scale and
approximate the matching computation. 

\subsection{Definition}
\label{sec:alg}

{\bf Setup.} We are given a clustering algorithm $\bC$ whose input is a finite metric space $X$ with $|X|>C$ and whose output is a finite metric space $X'$ together with a surjective (point-set) function  $p_X\colon X\to X'$, satisfying, for all $x\in X$:

1. $|X'|=C\in\mathbb N^*$. 
\ \ \ \ 2. $|p_X^{-1}(x)|$ is roughly constant.



{\bf Recursive scheme.} We now present our recursive
decomposition algorithm, that we call \GWREC\ and detail in Algorithm~\ref{alg:gwrec}, for producing a matching between $X$ and $Y$.
%
%

\begin{algorithm}[tb]
\label{alg:gwrec}
\begin{algorithmic}
\STATE {\bfseries Input:} dataset $X$, dataset $Y$
\STATE {\bfseries Output:} matching $\Gamma:X\rightarrow Y$
\STATE {\bfseries Parameters:} (black box) clustering algorithm $\bC$, 
number of clusters $C\ll |X|,|Y|$,
threshold $T\ll|X|,|Y|$, 
(black box) matching algorithm $\bM$
\IF{$|X|$ or $|Y|$ less than $T$}
  \STATE $\Gamma[X,Y]\leftarrow \bM(X,Y)$;
\ELSE
\STATE $X'$, $p_X\leftarrow \bC(X)$; $Y'$, $p_Y\leftarrow \bC(Y)$; \COMMENT{ Use $\bC$ to compute centroids $X'$, $Y'$ with $|X'|=|Y'|=C$, as well as maps $p_X:X\to X'$, $p_Y:Y\to Y'$ } 
\STATE $f\leftarrow \bM(X',Y')$; \COMMENT{Use $\bM$ to match $X'$ and $Y'$}
\STATE $X'\leftarrow \{x_1,\cdots,x_C\}$; $Y'\leftarrow \{y_1,\cdots,y_C\}$; \COMMENT{Enumerate the elements of $X'$ and use consistent enumeration (under the match $f$) for 
$Y'$. Let $\{x_i\}$ and $\{y_i\}$ denote these enumerations. 
}
\FOR{$i$ in $1,\cdots,C$}
\STATE return \GWREC$(p_X^{-1}(x_{i}),p_Y^{-1}(y_{i}))$; 
\ENDFOR
\ENDIF
\caption{\GWREC}
\end{algorithmic}
\end{algorithm}

%
At the conclusion of this algorithm, we obtain a matching $\Gamma$ between $X$ and $Y$; this is the output of the algorithm, along with the induced approximation of the Gromov-type distances. 
In practice, \GWREC\  is stochastic since it depends both on the specifics of the clustering algorithm $\bC$ and on random properties of the optimization solvers for the matching algorithm $\bM$.  As a consequence, the correct way to proceed is perform a search over the parameter space, possibly even searching over different clustering algorithms---each matching comes with a score, which is the associated distortion measure, allowing us to automatically hill-climb to find the best matching.

{\bf Associated code.} Our code is publicly available\footnote{\url{https://github.com/MathieuCarriere/mrec}}. 
It contains $K$-means
and Voronoi partitions (with randomly sampled
germs) for the black box clustering function $\bC$, as well as
entropy-regularized~\cite{Cuturi2013} (based on the \texttt{POT}
Python
package\footnote{\url{https://pot.readthedocs.io/en/stable/}}) and
the semidefinite relaxation~\cite{Villar2016} of the Gromov-Hausdorff distance for the potential black box matching function $\bM$. Moreover, it has been designed so that incorporating new clustering and matching functions is as easy as possible.

\subsection{Theoretical properties}
\label{sec:theo}
We begin by giving some generic properties of \GWREC.  If we let
$\alpha(T)$ denote the time required for the black box matching
algorithm to compute a matching on metric spaces of size smaller than
$T$, then when using $C$ clusters the time to execute \GWREC\ on datasets 
with $N$ points is $\alpha(T) \log_C \frac{N}{T}$.  That is, the
running time is basically determined by $\alpha(T)$.

Next, we turn to the question of how good an approximation we can
expect to get from \GWREC.  Recall that $X' \subset X$ is an
$\epsilon$-net if for every point $x \in X$ there is a point $x' \in
X'$ such that $\partial_X(x,x') < \epsilon$.  A basic observation is that if
$X'$ is an $\epsilon$-net of $X$, the evident induced matching that
assigns each point of $X$ to a point of $X'$ within distance
$\epsilon$ has discrepancy bounded above by $2\epsilon$.
The triangle inequality then implies that if $X' \subset X$ is an
$\epsilon_1$-net and $Y' \subset Y$ is an $\epsilon_2$-net, then the
optimal matching of $X'$ and $Y'$ induces a matching of $X$ and $Y$
that has distortion within $2(\epsilon_1 + \epsilon_2)$ of the
optimal.

It is helpful to reformulate this in terms of clustering algorithms.
Let $\bC$ be a clustering algorithm, and $r>0$.  Say that a finite metric
space $X$ is {\em $r$-clustered} if the output of $\bC$ on $X$
satisfies the condition that for all clusters, all elements in the
cluster are within distance $r$ of each other.  Notice that this in
particular means that choosing a point in each partition produces an
$r$-net.

\begin{prop}
Suppose $X'\subset X$, $Y'\subset Y$.  Given a relation
$R'\subset X'\times Y'$ and retractions $f\colon X\to X'$ and $g\colon Y\to Y'$, consider the derived relation $R\subset X\times Y$ defined by $\{(x,y)\,:\, (f(x),g(y))\in R'\}$.  Suppose that for all $x\in X$, $\partial_{X}(x,f(x))<r$ and likewise for $Y'$.  Then $\dis(R)<\dis(R')+4r$.
\end{prop}

\begin{proof}
For $(x_{1},y_{1}),(x_{2},y_{2})\in R$, we have
$(f(x_{1}),g(y_{1})),(f(x_{2}),g(y_{2}))\in R'$. Thus, we have, using the triangle inequality:
$\partial_{X}(x_{1},x_{2}) 
\leq \partial_{X}(x_{1},f(x_{1}))+\partial_{X}(f(x_{1}),f(x_{2}))
+\partial_{X}(f(x_{2}),x_{2})$ and 
$\partial_{Y}(y_{1},y_{2}) 
\geq -\partial_{Y}(y_{1},g(y_{1}))+\partial_{Y}(g(y_{1}),g(y_{2}))
-\partial_{Y}(g(y_{2}),y_{2})$.

This leads to:
$|\partial_{X}(x_{1},x_{2})-\partial_{X}(y_{1},y_{2})|<
|\partial_{X}(f(x_{1}),f(x_{2}))
-\partial_{Y}(g(y_{1}),g(y_{2}))|+4r$.
\end{proof}
This immediately implies the following corollary.

\begin{cor}
Let $X$ and $Y$ be $r$-clustered for $\bC$ with the number of
$\bC$-clusters $\leq C$, and let $X'$ and $Y'$ denote the metric
spaces of clusters generated by $\bC$.  Then {\em \GWREC} (computed
with the exact Gromov-Hausdorff distance as the matching function
$\bM$) produces a matching between $X$ and $Y$ with distortion at most
$\dis(R) + 4r$, where $R$ is the correspondence realizing 
$d_{GH}((X', \partial_X'), (Y', \partial_Y'))$.
\end{cor}

Now, a very natural question to ask is for a metric space with some
kind of dimension restriction, how many clusters are needed to obtain
an $r$-clustering.  A standard way of encoding the dimension of a
metric space is via {\em doubling dimension}.  Recall that a metric
space $(X,\partial_X)$ has doubling dimension $d$ if every ball of
radius $\epsilon$ can be covered by $2^d$ balls of radius
$\epsilon/2$.  For example, Euclidean space $\bR^n$ has doubling
dimension $O(n)$.  The following standard result now follows from an
easy counting argument. 

\begin{prop}
Let $(X,\partial_X)$ be a metric space with doubling dimension $d$.
Then there exists an $r$-clustering with approximately
$\left(\frac{M}{r}\right)^{d}$ clusters, where $M = \diam(X)$. 
\end{prop}

Here the approximation in the statement comes from issues of rounding
logarithms. 
%
%
These results give quite pessimistic estimates, but they are useful
for providing a sense of lower bounds on the performance to expect
from  \GWREC\ and its recursive decomposition.

The analyses above focus on the situation in which the coarse
approximation of the data that we get from partitioning
is close in the Gromov-Hausdorff metric to the original data.
(Analogous analyses can be performed using other metrics.)  We now explain a situation in which we can expect the
algorithm to work well even when the partition approximation is in this sense far from the original data.

Consider the toy model where $X$ and $Y$ are both finite samples
from a metric space $Z$ with the property that $Z$ can be partitioned
into $K$ clusters $\{C_i\}$ such that the minimum Hausdorff distance
$\delta$ between any pair of clusters $C_i$ and $C_j$ is substantially
larger than the maximum diameter of the clusters.  Specifically,
assume that  
\[
\delta = \min_{i,j} d_H(C_i, C_j) > 2 \max_i \diam(C_i) =
\eta.
\]
First off, note that in this case, the optimal clustering can be found
by the $K$-means algorithm; each point in a cluster is closer to the
centroid than to any point outside the cluster.  In addition, we know
that the expected number of samples required to to obtain a point in
each cluster is bounded above by $\left(\min_i (\# C_i)\right)^{-1}
\ln (K)$; this follows from a generalized coupon-collector analysis.
Let us assume that we have chosen enough points so that with high
probability we indeed have at least one sample from each cluster.
In this situation, the optimal matching between $X$ and $Y$ is
realized by a matching of the clusters.

\begin{prop}
Under the hypotheses above, the optimal matching of $X$ and $Y$
permutes the clusters.
\end{prop}

\begin{proof}
Any matching that separates points within a cluster (i.e., assigns
points $x_1, x_2 \in C_i \subset X$ to distinct clusters of $Y$) or
merges clusters will have discrepancy bounded below by $\delta -
\eta > \eta$, whereas a matching that permutes the clusters
will have discrepancy bounded above by $\eta$.
\end{proof}

An immediate corollary is that \GWREC\ will find the
optimal matching of the top-level clusters with high probability. Since we are only using metric information, our matchings are non-unique up to
symmetry in the data; if we assume that the inter-cluster distances
are unique, the matching returned is the unique matching that
identifies clusters to themselves.

\subsection{Discussion}

One of the most useful features of the \GWREC\ framework is that because
matchings can be evaluated in terms of their distortion, we can search
over parameter space to find the minimal distortion embedding.  For
example, it is reasonable to try a variety of clustering algorithms
and numbers of centers in order to ascertain which parameters lead to
the best approximation of the optimal matching.  Moreover, we can even
try different matching procedures.  This in particular means that the
algorithm can tolerate a substantial amount of uncertainty about the
decomposition of the best matching, as long as there exists a
partition and allowable resolution for which there is a good
matching.

Another useful aspect of this is that we can easily make \GWREC\ robust
to outliers by subsampling cross-validation; specifically, we can
repeatedly subsample the two metrics spaces $X$ and $Y$, perform the matchings on the subsamples, and identify points
that result in high distortion matchings.

Finally, although we have discussed the matching algorithm in terms of
a partition of the dataset, instead it is entirely reasonable to
imagine an overlapping cover of the datasets used for matchings.  We
only mention this refinement here; we intend to return to this in
future work.

\section{Applications}
\label{sec:appli}

\begin{figure*}
    \centering
    \begin{tabular}{p{5\textwidth/16} p{5\textwidth/16} p{5\textwidth/16}}
    \includegraphics[width=5\textwidth/16]{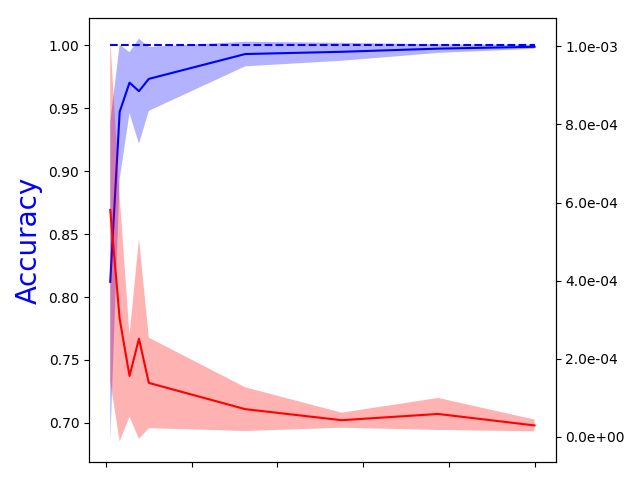} &
    \includegraphics[width=5\textwidth/16]{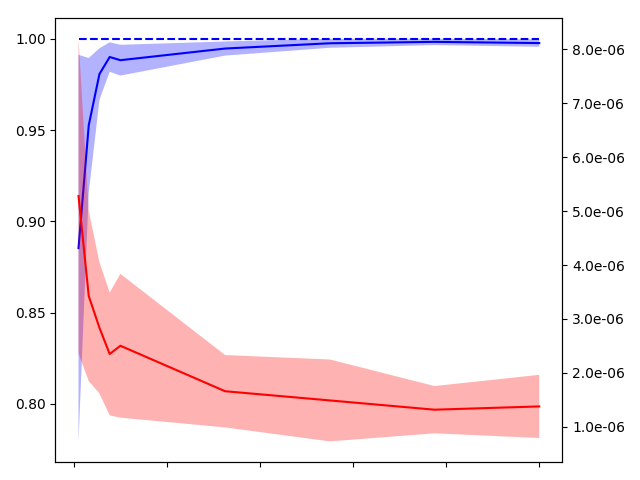} &
    \includegraphics[width=5\textwidth/16]{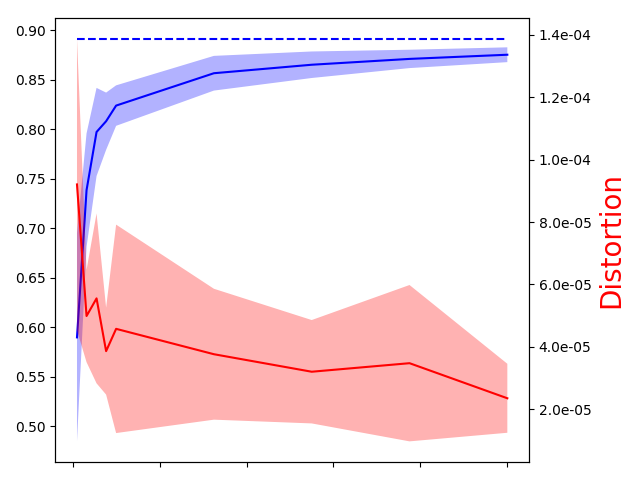} \\
    \includegraphics[width=5\textwidth/16]{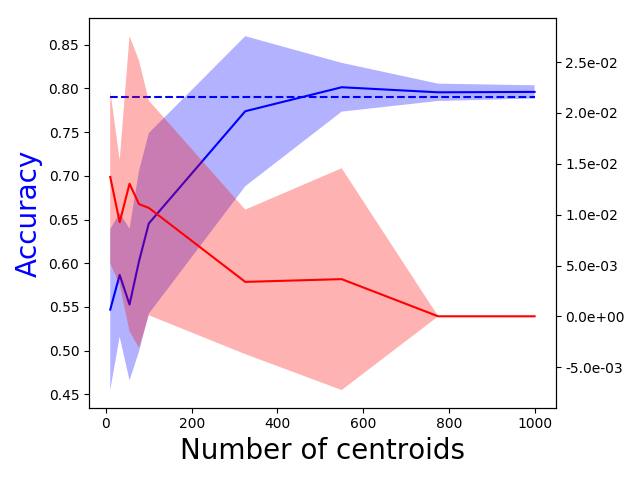} &
    \includegraphics[width=5\textwidth/16]{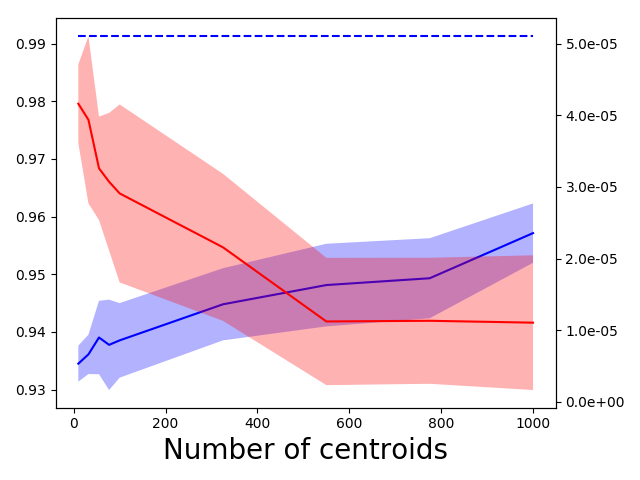} &
    \includegraphics[width=5\textwidth/16]{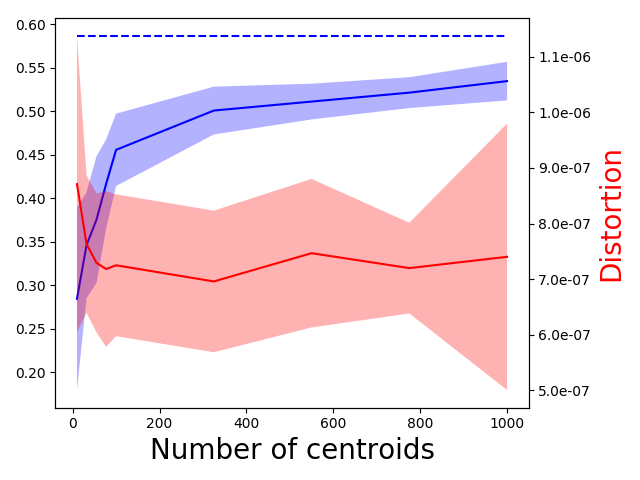}
    \end{tabular}
    \caption{Mean and variance of the accuracy and distortion of the matchings computed by \GWREC\ (for the parameters inducing the best matching) w.r.t. the number of centroids used for \texttt{Synth} (upper left), \texttt{Synth+} (upper middle), \texttt{PBMC} (upper right), \texttt{HEMA} (lower left), \texttt{BDTNP} (lower middle) and \texttt{BRAIN} (lower right). The dashed line corresponds to the accuracy obtained without doing recursion.}
    \label{fig:acc}
\end{figure*}

In this section, we demonstrate the use of \GWREC\ by applying it on synthetic and real biological
single-cell datasets with ground-truth labels. Overall, we show that the recursion
procedure used in \GWREC\ is highly efficient, in the sense that the
computed matchings are on par, or slightly better (as
measured by the accuracy of label transfer under the computed
matchings), than the ones obtained without doing any recursion, with
significantly smaller running times. We first provide
proof-of-concepts of \GWREC\ for synthetic and real data in
Sections~\ref{sec:synth} and~\ref{sec:pbmc}. 
Then, we
present applications of \GWREC\ for correcting batch effects on a
single-cell datasets in Section~\ref{sec:hema} and aligning across
different measurement modalities in Sections~\ref{sec:bdtnp} and~\ref{sec:brain}.  
Whereas the datasets in Sections~\ref{sec:hema} and~\ref{sec:bdtnp} can be handled directly
the dataset presented in Section~\ref{sec:brain} is
too large to be handled with usual matching algorithms, making \GWREC\
as far as we know the only algorithm suitable for working at scale
with the recent breakthroughs in single-cell data generation.

{\bf Single-cell datasets.} Single-cell RNA sequencing has been a
major event in the fine characterization of many biological systems at the molecular level, since it allows biologists to study the gene
expression levels of thousands of individual cells, leading to a clearer understanding of many biological processes. For instance, clustering the expression of single-cell data identifies sets of cells with common properties and function, and the relation between these clusters can reflect transitional states.  In these datasets, each cell
is represented as a point in Euclidean space, in which every dimension
is associated to the expression of a gene of interest, and each single cell coordinate
contains the {\em expression} of the corresponding gene in that cell
(that is, how often is this gene transcribed in the cell, as measured
with the number of transcript fragments that are detected in the cell). Hence, raw single-cell datasets are usually given as
integer matrices. They often require preprocessing; we will use the Python packages
\texttt{Scanpy}~\cite{Wolf2018}
and
\texttt{Randomly}~\cite{Aparicio2018}.

{\bf Batch effects.} Single-cell RNA sequencing has resulted in a huge
increase in data availability. However, many methodological issues with using these datasets remain.  Notably, datasets computed in different labs or with
different protocols usually exhibit unwanted biases, often
referred to as {\em batch effects}~\cite{Haghverdi2018}. Matching, or
aligning, these datasets in order to integrate them together in a
robust way is thus a problem of major importance, and optimal
transport with the Gromov-Wasserstein distances has recently been shown to
be a promising solution in this application~\cite{Forrow2019}. In
Section~\ref{sec:hema}, we present an example of batch effect
correction with \GWREC.


{\bf Different measure modalities.} Another class of alignment problems
comes from situations where we have measurements of the same dataset
coming from different kinds of technologies; for example, measuring
expression and chromatin accessibility.  Understanding the amount of alignment
that can be expected between these two kinds of measurements is an
extremely interesting problem.  In Sections~\ref{sec:bdtnp} and~\ref{sec:brain}, we describe
applications to two such datasets.

{\bf Parameters and results.} Each of the datasets in
Sections~\ref{sec:synth},~\ref{sec:pbmc},~\ref{sec:hema}
and~\ref{sec:brain} presented below is comprised of two groups with the same list of
associated ground-truth labels  (except for \texttt{BDTNP}; see the details in Section~\ref{sec:bdtnp} below). On each of these datasets, we use a
Voronoi partition computed with randomly sampled centroids\footnote{We also tried using $K$-means. However, we did not observe noticeable changes in accuracy, but substantially longer running times, so we stick to randomly sampled centroids.} as the clustering algorithm $\bC$, a threshold $T=10$, and we searched over 20 runs with
matchings coming from both the entropy-regularized Gromov-Wasserstein
distance (with $\epsilon$ ranging from $10$ to $10^{-4}$ in a log-scale) and the SDP relaxation as the matching algorithm $\bM$ for
\GWREC. 
Note that, for all datasets except \texttt{BDTNP} in
Section~\ref{sec:bdtnp} and \texttt{BRAIN} in
Section~\ref{sec:brain}, the two groups belong to the same common
space, so we can use the Wasserstein distance as well.  We also
compute matchings for an increasing number $C$ of clusters (see
Algorithm~\ref{alg:gwrec}) ranging from $10$ to $10^3$ in a log-scale. For each
computed matching, we measure three quantities: the running time of
\GWREC, the distortion of \GWREC\ (see Section~\ref{sec:GH}), and its
accuracy, given by computing the fraction of points in the first group
whose associated points under the matching given by \GWREC\ share the
same label (in the second group). We show in Figure~\ref{fig:acc} the evolution of accuracy and distortion 
w.r.t. the number of centroids for all datasets (see Appendix, Figure~\ref{fig:time}, for running times). We finally provide in Table~\ref{tab:res} an overall comparison with directly computing  the matchings (i.e.,
without doing any recursion). In this table, we provide the accuracies of the matchings with lowest distortion.
All computations were performed on an AWS machine with a Xeon Platinum 8175 processor.
See also Appendix, Figures~\ref{fig:params} and~\ref{fig:synth}, for more illustrations and examples of parameter tuning. 

\subsection{Matching of simulated datasets \texttt{Synth(+)}}\label{sec:synth}
We first apply \GWREC\ on a synthetic dataset sampled from a
mixture of three Gaussian probability distributions located at
different positions in the Euclidean plane $\mathbb{R}^2$. More specifically, we drew
$6,000$ samples that we eventually split into two groups. 
Accuracy is then measured as the
ability of \GWREC\ to match points sampled from the same Gaussian
probability distributions together.
We also run an experiment with $60,000$ samples (called \texttt{Synth+}),

\subsection{Matching of subsampled datasets \texttt{PBMC}}\label{sec:pbmc}
In our second example, we test the performance of \GWREC\ on a
single-cell transcriptional dataset. We collected the data presented
in~\cite{Kang2018}, which is comprised of single peripheral blood
mononuclear cells (PBMC), with eight different cell types (CD4, CD8,
B-cell, NK-cell, dendritic cells (Ddr), CD14, monocytes (Mnc) and
megakaryocytes (Mgk)). 
We preprocessed the cells with \texttt{Randomly}~\cite{Aparicio2018}
to create a dataset of $6,573$ cells in $1,581$ dimensions (each of
which representing a highly variable gene of interest). We then split
the data into two groups and measured the ability of \GWREC\ to match
together cells with the same types coming from different groups. 


\subsection{Matching of different protocols \texttt{HEMA}}\label{sec:hema}
In our third application, we study an example of batch effects in
single-cell transcriptional data for which optimal transport has already proved
useful~\cite{Forrow2019}. The data is comprised of two groups of
$2,729$ and $813$ single hematopoietic cells respectively, with
$3,467$ associated genes and three different cell types (common
myeloid  progenitors (CMP), granulocyte-monocyte  progenitors (GMP)
and megakaryocyte-erythrocyte progenitors (MEP)). These groups were
generated with the SMART and MARS protocols respectively, as explained
in~\cite{Haghverdi2018}. As mentioned above, even though the
dimensions are in correspondence, the fact that different protocols
have been used introduces a bias, or {\em batch effect}, in the data,
We preprocessed the data using the method presented
in~\cite{Zheng2017} and available in the Python package
\texttt{Scanpy} and then reduced the number of dimensions to 50 with
PCA. 

\subsection{Matching of different modalities \texttt{BDTNP}}\label{sec:bdtnp}
In this section, we study 3,039 single cells coming from Drosophila embryos, as presented in~\cite{Nitzan2018}. Spatial coordinates in the 2D tissues were obtained with fluorescence in situ hybridization (FISH), as well as the expression of 84 marker genes, leading to spaces with different dimensions. This means that usual Wasserstein distance cannot be used here.
We followed the exact same procedure as in~\cite{Nitzan2018}, i.e., pairwise distances in both spaces were computed with nearest-neighbor graph, so as to approximate the geodesics from the manifolds on which the datasets have been sampled (following the manifold assumption made in~\cite{Nitzan2018}). Finally, the ground truth is here given by the true correspondence between the cells. Hence, we chose to measure accuracy with the (normalized) area under the curve that shows, for each radius $r>0$, the fraction of cells $\alpha(r)$ (in expression space) whose corresponding cell (in tissue space) under the computed matching is at distance at most $r$ from the true corresponding cell (which is a common practice  in computer graphics~\cite{Carriere2015}).   

\subsection{Matching of different modalities \texttt{BRAIN}}\label{sec:brain}

In our last example, we focus on a dataset of single cells sampled
from the human brain with eight different cell types (astrocytes (Ast),
endothelial cells (End), excitatory neurons (Ex), inhibitory neurons
(In), microglia (Mic), oligodendrocytes (Oli) and their precursor
cells (OPC)), presented in~\cite{Lake2017}.  
The challenge of this dataset is
two-fold. First, the dimensions of the two groups are not in
correspondence (as in Section~\ref{sec:bdtnp}): the gene expressions were measured in the
first group, while the second group contains the {\em DNA region
  accessibilities} (that is, whether given regions in the DNA of the
cell nucleus are open, i.e., accessible for transcription, or
not). Second, the datasets are too large to be handled with usual
matching techniques: 
the first group contains $34,079$ cells, while the second
one contains $27,906$  cells.  
We preprocessed the first group with the method presented
in~\cite{Zheng2017} and available in the Python package
\texttt{Scanpy}, and the second one with TF-IDF. Then, we applied PCA
on each group separately to reduce the number of dimensions to
50. Note that since the dimensions are not in correspondence anymore,
we have to use the Gromov-Wasserstein distance.  We also added
a standard Wasserstein cost by using a
binary cost matrix containing the correspondence between the genes and
the DNA regions 
(that is, the $(i,j)$ entry of
the matrix is one if the genomic coordinates of gene $i$ intersects
DNA region $j$).






\subsection{Results and comparison with no recursion.} 

It can be seen from  Figure~\ref{fig:acc} 
that, as expected, accuracy 
generally tends to increase while distortion decreases with the number of centroids. Note also that, for the \texttt{BRAIN} dataset, the variance of the distortion is pretty large, and its decrease is not obvious. This large variance is also the reason why the mean accuracy (observed on the plot) is actually much smaller than the one corresponding to the matching with lowest distortion (shown in Table~\ref{tab:res})---note that this is actually true (to a certain extent) for all datasets. Note also that there is clear gap in Figure~\ref{fig:acc} between the accuracy of matchings computed with \GWREC\ on \texttt{BDTNP} and the baseline. This suggests that convergence has not occurred yet and more centroids are needed. 
Indeed, the accuracy and running times provided in Table~\ref{tab:res}  were obtained by allowing the number of centroids to go up to 3,000 (see Appendix, Figure~\ref{fig:bdtnp+}).



Finally, 
it can be seen from Table~\ref{tab:res} that recursion is mandatory
on \texttt{Synth+} and \texttt{BRAIN} for the matching computation to end in a
reasonable amount of time. Moreover, the results are generally on par (or slightly better) than the baseline with substantially lower running times. 
In the case of \texttt{HEMA}, the
accuracy is more significantly better\footnote{The accuracies for \texttt{HEMA} provided
  in this work are different from the ones presented
  in~\cite{Forrow2019} on the same dataset due to the following
  reason: in~\cite{Forrow2019} the authors provide the average of the
  accuracies associated to several random samples of $300$ cells in
  each group, whereas we measure the accuracy on the whole dataset.}.  This suggests that for some datasets it may be
the case that the recursive procedure may lead to better
approximations than the ``absolute'' algorithms.

\begin{table}
\begin{center}
\begin{tabular}{|l|cc||cc|}
\cline{2-5}
\multicolumn{1}{c}{}	& \multicolumn{2}{|c||}{Accuracy (\%)} & \multicolumn{2}{|c|}{Time (s)} \\
\cline{1-5}
\multicolumn{1}{|c|}{Dataset}   & \GWREC\ & No rec. & \GWREC & No rec. \\
\cline{1-5}
	\multicolumn{1}{|c|}{$\texttt{Synth}$}  & \bf{100} & \bf{100}  & \bf{14.1} & 505.8 \\
	\multicolumn{1}{|c|}{$\texttt{Synth+}$} & \bf{100} & \bf{100} & \bf{126.7} & 14,325.6 \\
	\multicolumn{1}{|c|}{$\texttt{PBMC}$}   & \bf{89.9} & 89.5 &  \bf{29.1} & 77.3\\
	\multicolumn{1}{|c|}{$\texttt{HEMA}$}   & \bf{84.9} & 79.5 &  \bf{20.5} & 119.0 \\
	\multicolumn{1}{|c|}{$\texttt{BDTNP}$}  & \bf{99.2} & 99.1  & \bf{73.2} & 508.5 \\
	\multicolumn{1}{|c|}{$\texttt{BRAIN}$}  & \bf{59.6} & 58.6 &  \bf{125.1} & 15,580.5 \\
\cline{1-5}
\end{tabular}
\end{center}
\caption{Results of \GWREC\ and its counterpart with no recursion.}
\label{tab:res}
\end{table}

\section{Conclusion}
\label{sec:conc}

In this article, we introduced \GWREC: a fast and versatile tool for
computing matchings on large-scale datasets.  It can use any black box
matching or clustering function defined by the user, and scale it with
a recursive scheme so as to be able to process datasets with very
large numbers of points in an efficient way.  Theoretical analysis
shows that the algorithm can in principle perform well.  We validated
its use with several applications in single-cell molecular analysis, achieving comparable or
better performance than the state-of-the-art on smaller datasets and
providing novel results on datasets too large to be aligned efficiently by any
other means.

\bibliography{biblio}
\bibliographystyle{alpha}

\newpage
\appendix

\section*{Appendix}

\subsection*{Additional figures}
In this appendix, we provide various additional plots. \\

\begin{itemize}
    \item In Figure~\ref{fig:time}, we display the running times of our experiments w.r.t. the number of centroids. Unsurprisingly, it increases with the number of centroids, with a rate that depends on the experiment.
    \item In Figure~\ref{fig:bdtnp+}, we confirm that \texttt{MREC} converges on the \texttt{BDTNP} experiment by showing the accuracy and distortion for a larger number of centroids than presented in the main article.
    \item In Figure~\ref{fig:params}, we present the dependency of accuracy, distortion and running times w.r.t. both entropy and number of centroids. Note that for small entropy values on \texttt{HEMA}, numerical errors occurred, leading to missing values for the distortion. 
    \item In Figure~\ref{fig:synth}, we provide illustrations of the datasets and computed matchings.
\end{itemize} 

\subsection*{Code}


We provide the code and data that we used to run the experiments (except for \texttt{BRAIN} and \texttt{Synth+} which are too large, although we provide a way to generate the data for \texttt{Synth+} in the \texttt{synth+\_matrix.py} file---just run \texttt{python synth+\_matrix.py} in a terminal) in the \texttt{experiments} folder available on this link: \\

\url{https://drive.google.com/open?id=16KZTlS3pDFW64QTfErJhaLeQX0kWVFc3} \\

The experiments can be reproduced by running \texttt{./launchS.sh} and \texttt{./launchnonS.sh} in a terminal. 
Since it is (very) long to run all experiments, we also provide a notebook \texttt{Example.ipynb} at the following url: \url{https://github.com/MathieuCarriere/mrec}. This notebook illustrates the use of \texttt{MREC} on the dataset \texttt{synth}. 
%
Dependencies for running both codes are \texttt{numpy}, \texttt{scipy}, \texttt{POT}, \texttt{scikit-learn}, \texttt{cvxopt}, \texttt{matplotlib} and (optionally) the Matlab engine API for Python. 

\newpage
\begin{figure}
    \centering
    \includegraphics[width=2\textwidth/7]{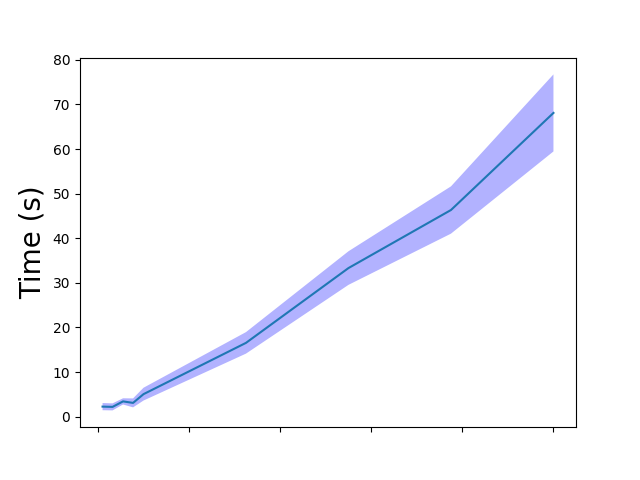}
    \includegraphics[width=2\textwidth/7]{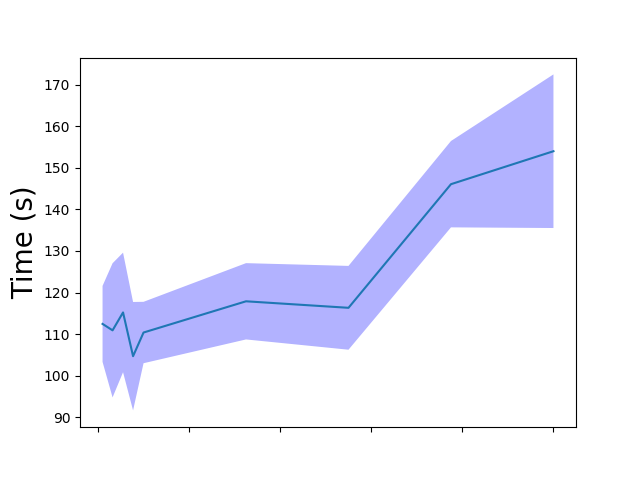}
    \includegraphics[width=2\textwidth/7]{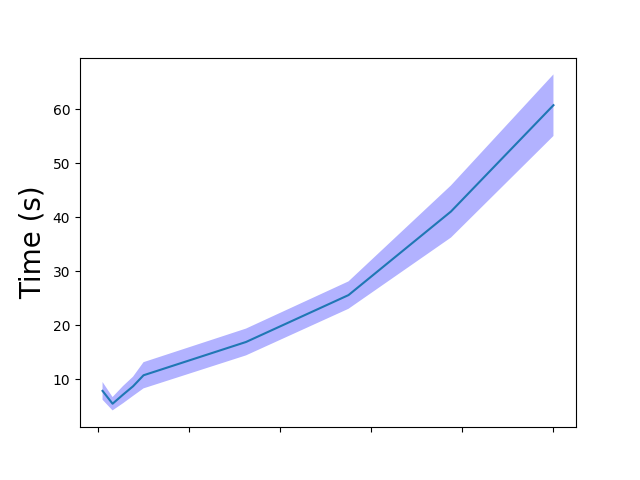}
    \includegraphics[width=2\textwidth/7]{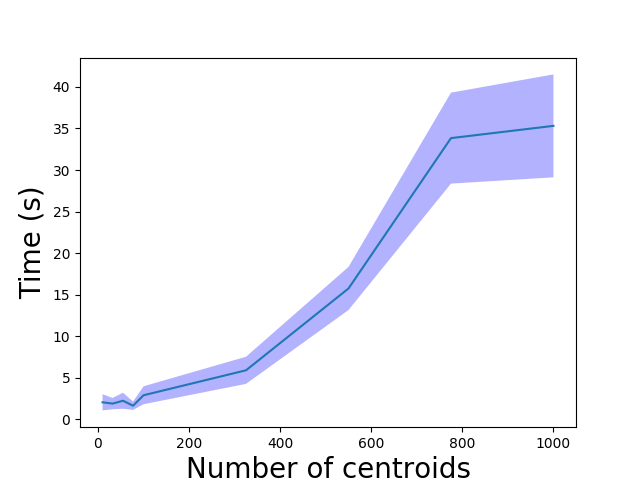}
    \includegraphics[width=2\textwidth/7]{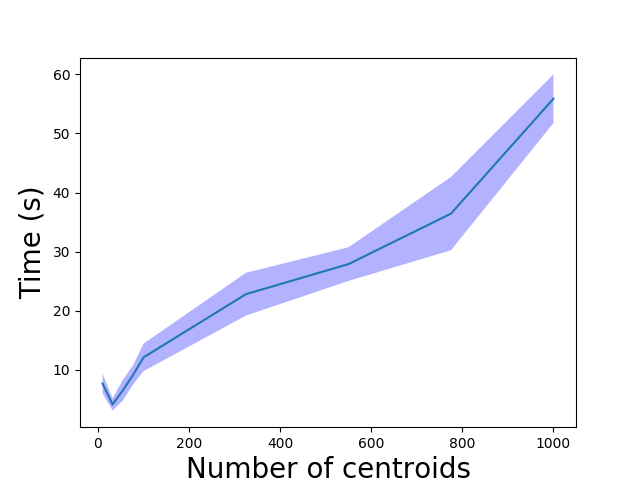}
    \includegraphics[width=2\textwidth/7]{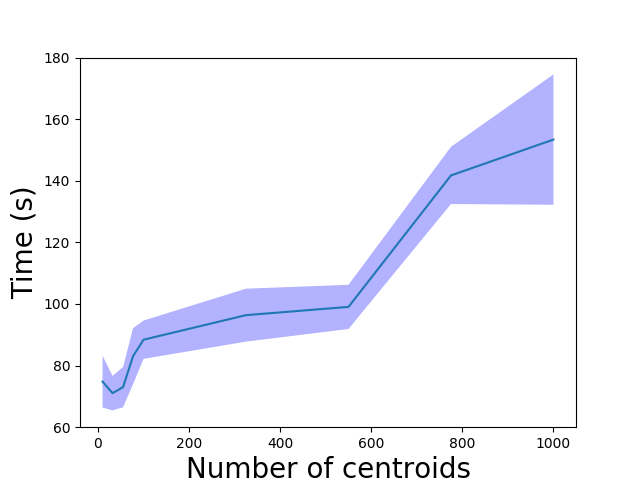}
    \caption{Mean and variance of the running time of the matchings computed by \GWREC\ (for the parameters inducing the best matching) w.r.t. the number of centroids used  for \texttt{Synth} (upper left), \texttt{Synth+} (upper middle), \texttt{PBMC} (upper right), \texttt{HEMA} (lower left), \texttt{BDTNP} (lower middle) and \texttt{BRAIN} (lower right).}
    \label{fig:time}
\end{figure}

\begin{figure}
    \centering
    \includegraphics[width=2\textwidth/5]{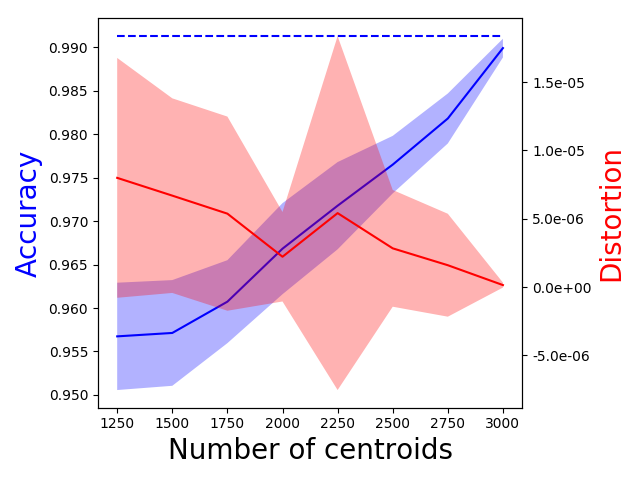}
    \includegraphics[width=2\textwidth/5]{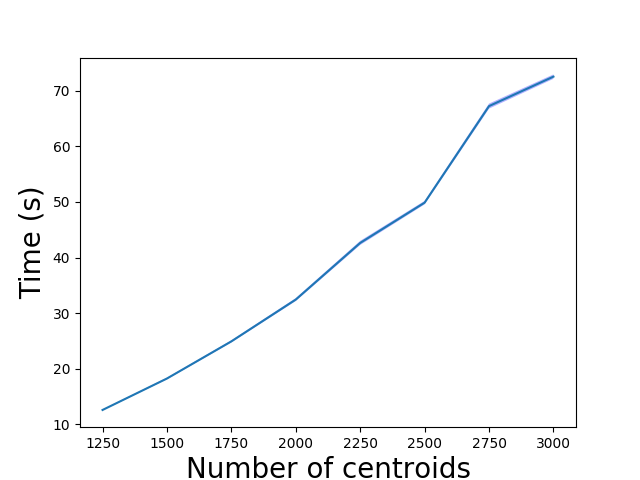}
    \caption{Mean and variance of the accuracy, distortion (left) and running time (right) of the matchings computed by \GWREC\ (for the parameters inducing the best matching) w.r.t. the number of centroids used  for \texttt{BDTNP}.}
    \label{fig:bdtnp+}
\end{figure}

\newpage
\begin{figure}
    \centering
    \includegraphics[width=2\textwidth/7]{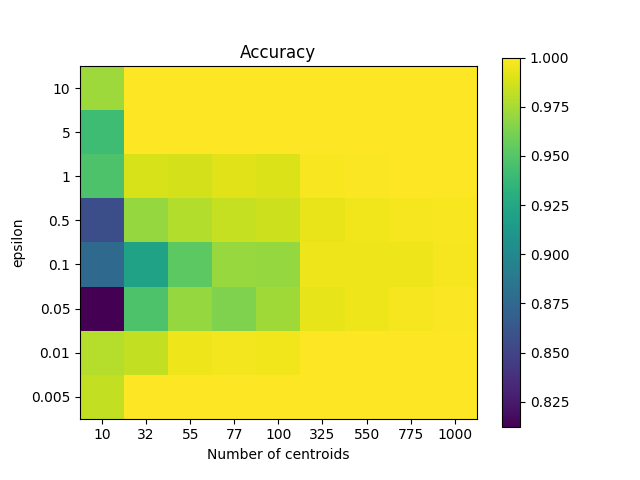}
    \includegraphics[width=2\textwidth/7]{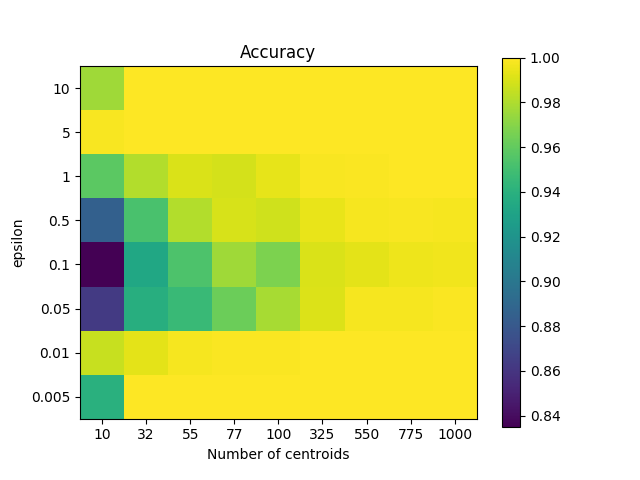}
    \includegraphics[width=2\textwidth/7]{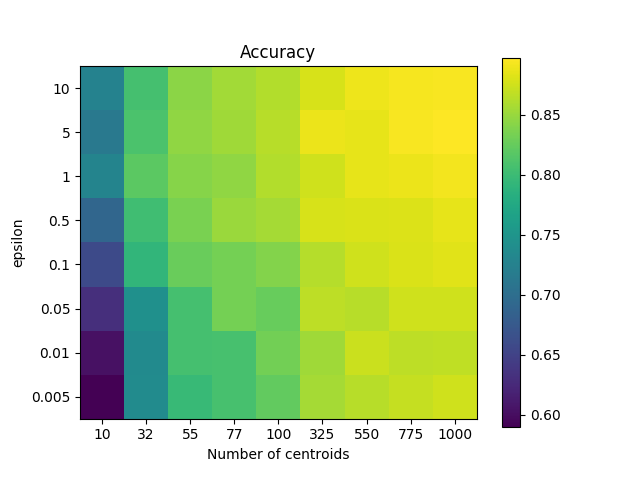}
    \includegraphics[width=2\textwidth/7]{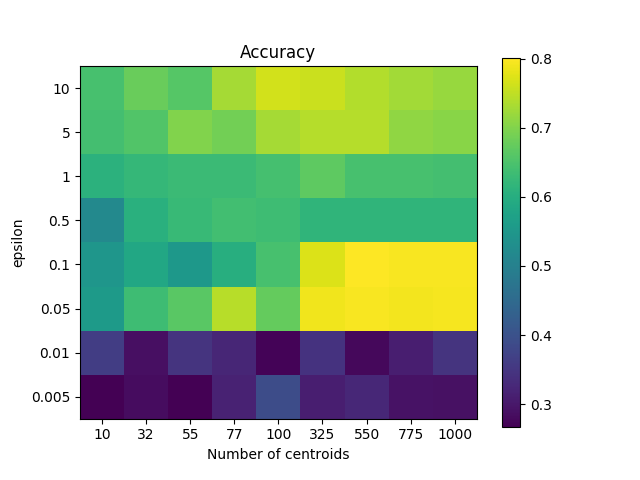}
    \includegraphics[width=2\textwidth/7]{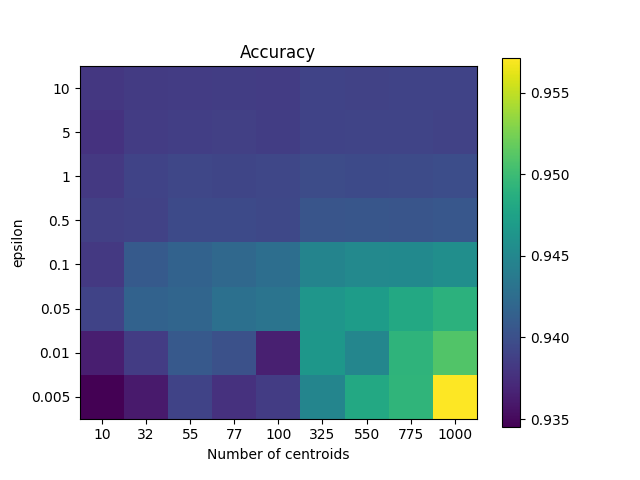}
    \includegraphics[width=2\textwidth/7]{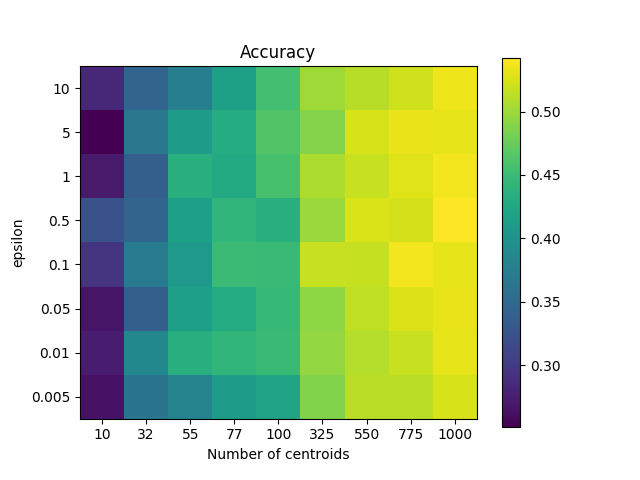}
    \includegraphics[width=2\textwidth/7]{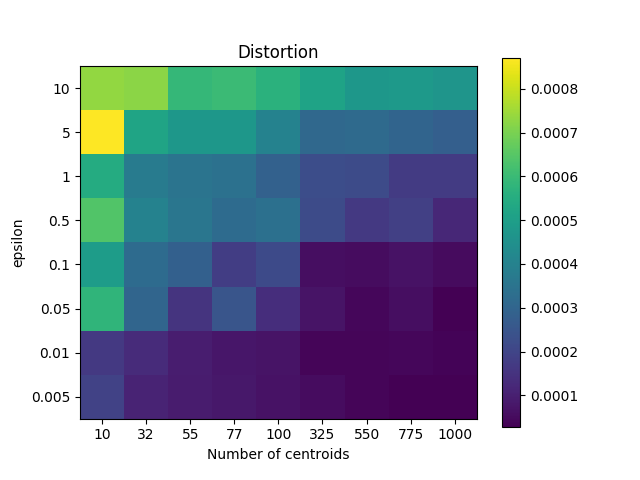}
    \includegraphics[width=2\textwidth/7]{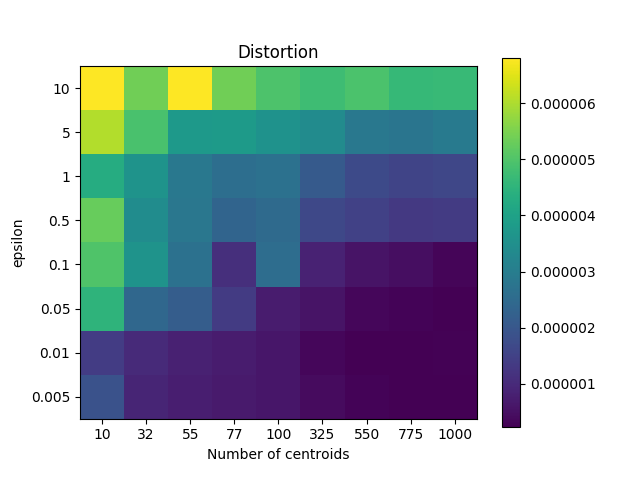}
    \includegraphics[width=2\textwidth/7]{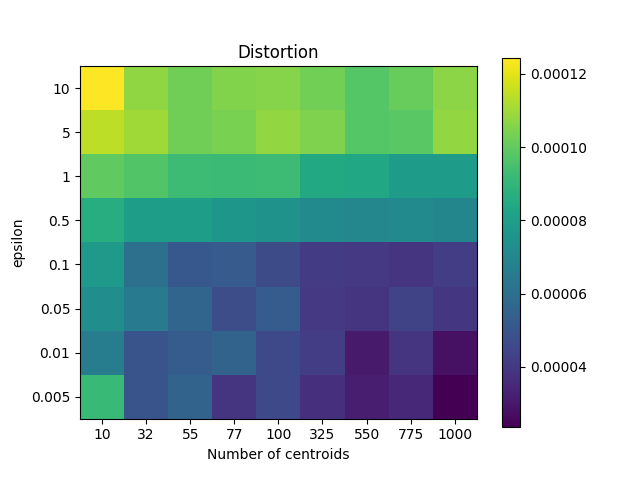}
    \includegraphics[width=2\textwidth/7]{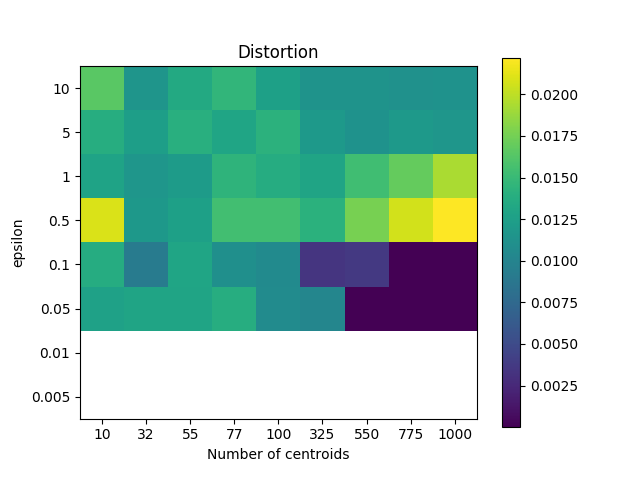}
    \includegraphics[width=2\textwidth/7]{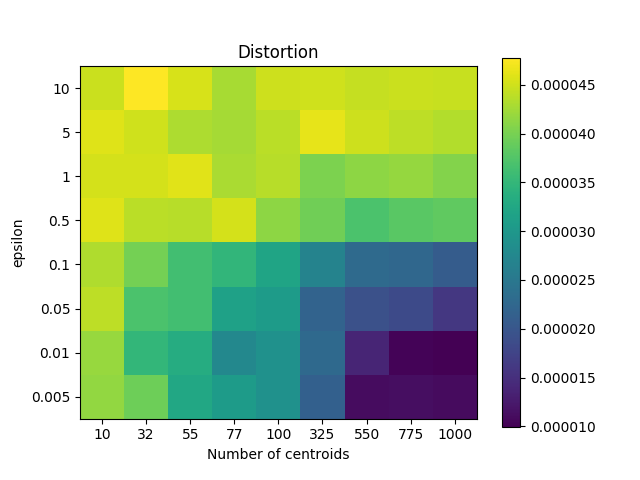}
    \includegraphics[width=2\textwidth/7]{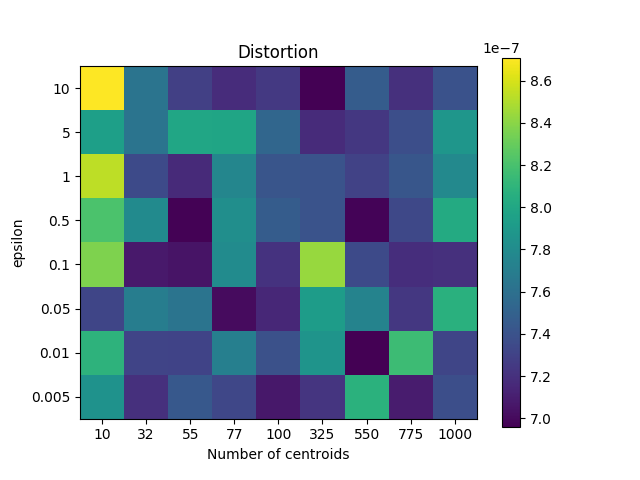}
    \includegraphics[width=2\textwidth/7]{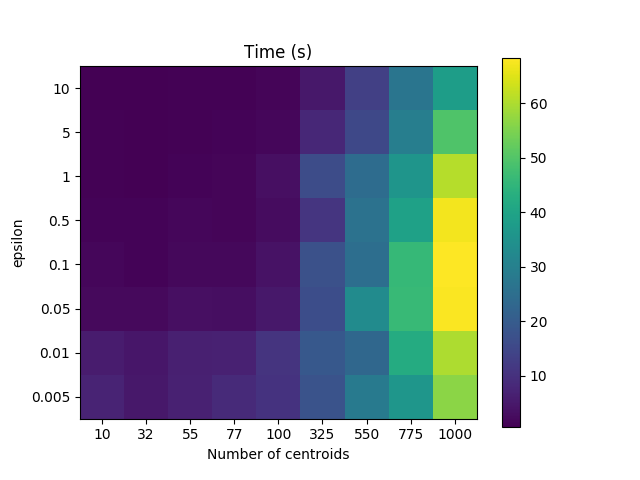}
    \includegraphics[width=2\textwidth/7]{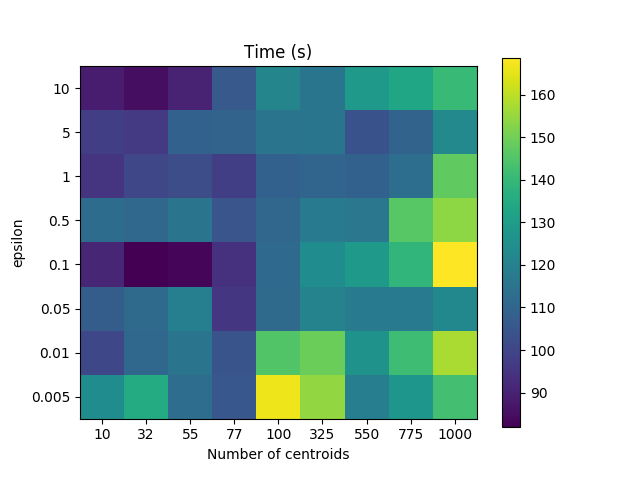}
    \includegraphics[width=2\textwidth/7]{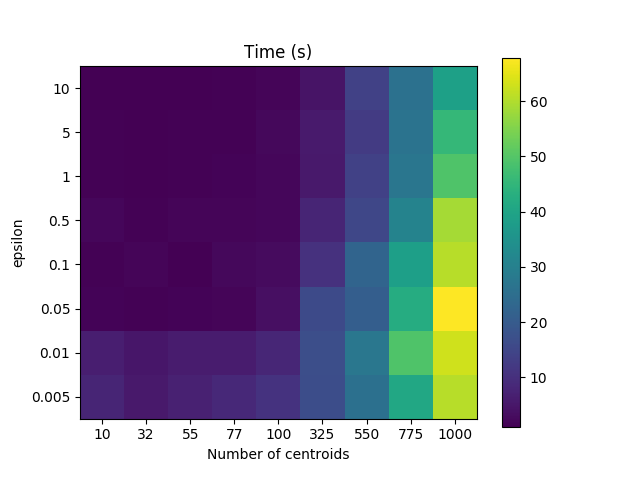}
    \includegraphics[width=2\textwidth/7]{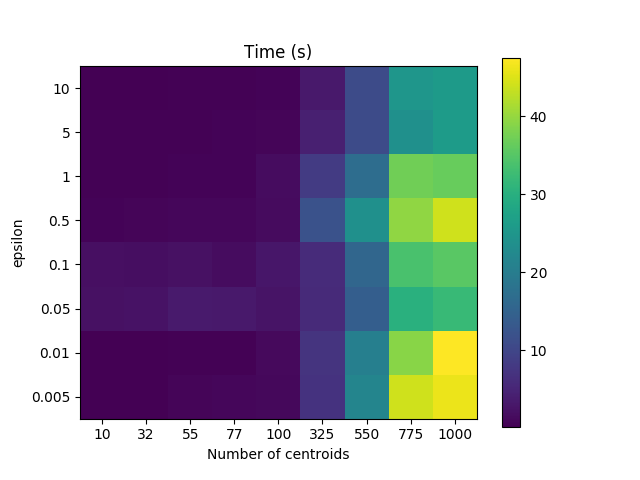}
    \includegraphics[width=2\textwidth/7]{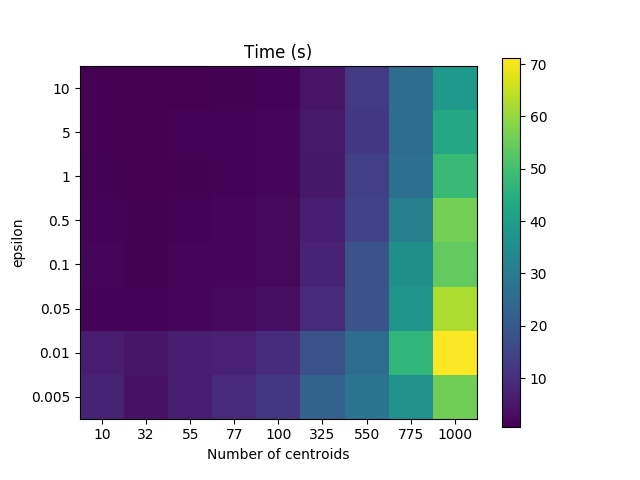}
    \includegraphics[width=2\textwidth/7]{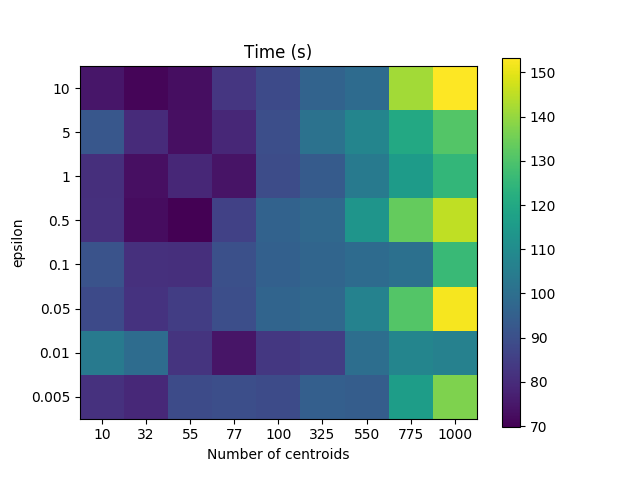}
    \caption{Example of parameter tuning for the $\epsilon$ entropy parameter for \texttt{Synth} (upper left), \texttt{Synth+} (upper middle), \texttt{PBMC} (upper right), \texttt{HEMA} (lower left), \texttt{BDTNP} (lower middle) and \texttt{BRAIN} (lower right.}
    \label{fig:params}
\end{figure}

\newpage
\begin{figure}[h!]
    \centering
    \includegraphics[width=6.5cm]{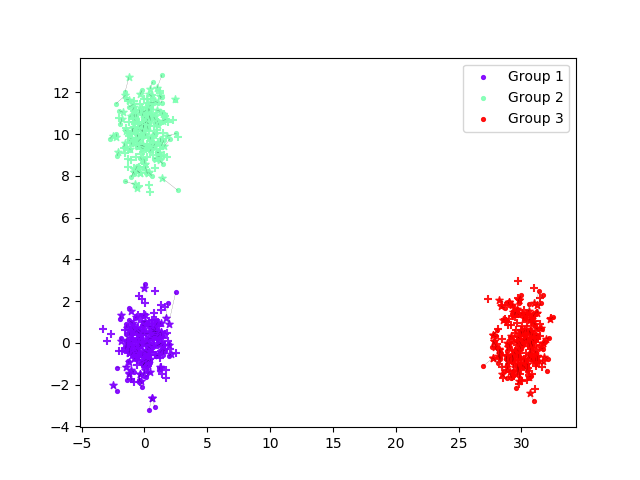}
    \includegraphics[width=6.5cm]{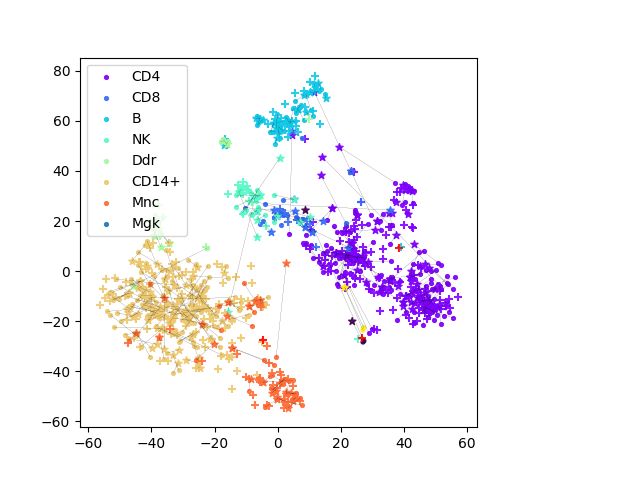}
    \includegraphics[width=6.5cm]{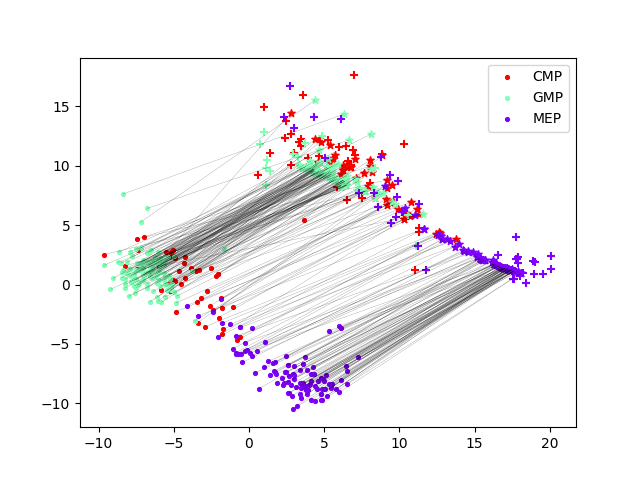}
    \includegraphics[width=6.5cm]{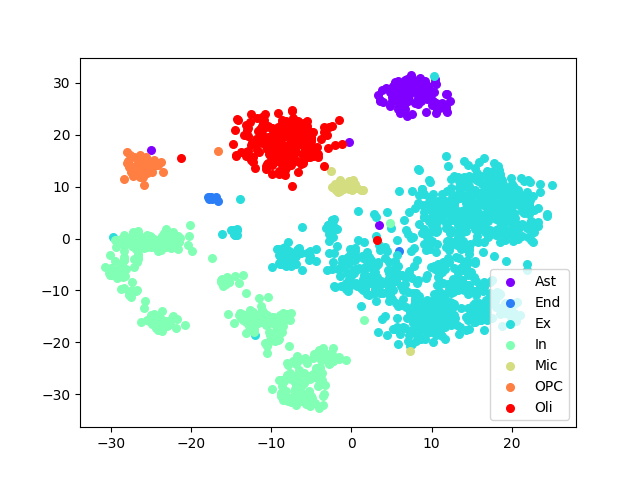}
    \caption{Upper left: Application of \GWREC\ on a synthetic data set of a mixture of Gaussians. The two input datasets are represented with dots and crosses, and \GWREC\ prediction (with 1,000 centroids) for each dot, that is, the weighted average of the second dataset computed with the probabilities given by \GWREC, is represented with a star. Upper right: Application of \GWREC\ on a data set of PBM cells. The two input datasets are represented with their first two principal components with dots and crosses, and \GWREC\ prediction (with 1,000 centroids) for each dot, that is, the weighted average of the second dataset computed with the probabilities given by \GWREC, is represented with a star. Lower left: Application of \GWREC\ on a data set of hematopoietic cells generated with the SMART and MARS protocols. The two input datasets are represented with their first two principal components with dots and crosses, and \GWREC\ prediction (with 1,000 centroids) for each dot, that is, the weighted average of the second dataset computed with the probabilities given by \GWREC, is represented with a star. Lower right: Application of \GWREC\ on a data set of human brain cells. Gene expression groups are visualized with the first two principal components.}
    \label{fig:synth}
\end{figure}

\end{document}